\documentclass[english]{IEEEtran}
\usepackage[T1]{fontenc}
\usepackage[latin9]{inputenc}
\usepackage{textcomp}
\usepackage{amsmath}
\usepackage{amssymb}
\usepackage{stackrel}
\usepackage{graphicx}

\makeatletter
\IEEEoverridecommandlockouts
\ifCLASSOPTIONcompsoc
\usepackage[caption=false,font=normalsize,labelfont=sf,textfont=sf]{subfig}
\else
\usepackage[caption=false,font=footnotesize]{subfig}
\fi
\usepackage{cite}
\usepackage{amsthm}
\newtheorem{theorem}{\textbf{Theorem}}

\usepackage{algorithm}
\usepackage{algorithmic}
\usepackage{amsmath}

\@ifundefined{showcaptionsetup}{}{%
 \PassOptionsToPackage{caption=false}{subfig}}
\usepackage{subfig}
\makeatother

\usepackage{babel}
\begin{document}
\title{Energy-Efficient Power Control for Multiple-Task Split Inference in
UAVs: A Tiny Learning-Based Approach}
\author{Chenxi~Zhao, Min Sheng, \IEEEmembership{Senior Member, IEEE,} Junyu
Liu, \IEEEmembership{Member,~IEEE,} Tianshu Chu, \\
and Jiandong Li, \IEEEmembership{Fellow,~IEEE}}
\maketitle
\begin{abstract}
The limited energy and computing resources of unmanned aerial vehicles
(UAVs) hinder the application of aerial artificial intelligence. The
utilization of split inference in UAVs garners significant attention
due to its effectiveness in mitigating computing and energy requirements.
However, achieving energy-efficient split inference in UAVs remains
complex considering of various crucial parameters such as energy level
and delay constraints, especially involving multiple tasks. In this
paper, we present a two-timescale approach for energy minimization
in split inference, where discrete and continuous variables are segregated
into two timescales to reduce the size of action space and computational
complexity. This segregation enables the utilization of tiny reinforcement
learning (TRL) for selecting discrete transmission modes for sequential
tasks. Moreover, optimization programming (OP) is embedded between
TRL's output and reward function to optimize the continuous transmit
power. Specifically, we replace the optimization of transmit power
with that of transmission time to decrease the computational complexity
of OP since we reveal that energy consumption monotonically decreases
with increasing transmission time. The replacement significantly reduces
the feasible region and enables a fast solution according to the closed-form
expression for optimal transmit power. Simulation results show that
the proposed algorithm can achieve a higher probability of successful
task completion with lower energy consumption. 
\end{abstract}

\begin{IEEEkeywords}
Power control, tiny learning, energy-efficient, multiple-task split
inference.
\end{IEEEkeywords}

\section{Introduction}

In recent years, there has been a remarkable surge in the proliferation
of unmanned aerial vehicles (UAVs), accompanied by the deployment
of diverse delay-sensitive applications on UAVs, e.g., surveillance
and exploration \cite{UAV-sur,UAV-ex}, where UAVs can send the data
through the satellite network or cellular network \cite{arxiv-back1,arxiv-back2}.
For expediting access to and processing of the data from UAVs, the
widespread implementation of artificial intelligence (AI) technologies
in UAVs is imperative. The concept of aerial AI refers to the implementation
of deep neural networks (DNN) on UAVs, thereby facilitating localized
provision of AI services that empower emerging applications at Internet
of UAVs. However, the deployment of aerial AI on UAVs is hindered
by constraints in energy and computing capabilities.

The split inference is a state-of-the-art AI architecture, wherein
the DNN is divided into device and server sub-models \cite{3-3,UAVsplit1,UAVsplit2}.
Utilizing the device sub-model, a UAV extracts features from raw data
samples and transmits them to a server. Subsequently, the server employs
these features to compute an inference result. The utilization of
split inference can significantly alleviate the computation load on
UAVs, thereby reducing energy consumption. Nevertheless, there are
some challenges associated with the application of split inference.
Specifically, ensuring the delay constraint becomes challenging under
direct data transmission, while transmitting data features after processing
significantly increases energy consumption. The delay constraint and
energy consumption of split inference poses a complex trade-off. Worsestill,
the coupling among multiple sequential tasks and the random time-varying
wireless channel further complicate the aforementioned issue. Therefore,
this work focuses on 1) how to choose whether to apply split inference
based on the current energy level of UAV and task delay constraint,
and 2) optimizing the resource allocation policy for multiple tasks
considering the random time-varying wireless channel. 

\subsection{Related Work and Motivation}

Extensive research has been conducted on the edge AI problem, and
comprehensive surveys on this topic can be found in \cite{1-1,1-2,1-3,1-4,arxiv2}.
In \cite{1-1}, authors propose an on-demand DNN collaborative inference
framework, wherein edge devices collaborate to achieve low latency
edge AI. Similarly, a multi-device edge computing framework is proposed
in \cite{1-2}, where the task execution of multi-device edge computing
is defined as a multi class classification problem. Moreover, a hardware-based
prototype and a software framework for optimizing a pre-trained YOLOv2
model is introduced in \cite{1-3} to achieve lower inference latency
for a single input by collating partial outputs generated by partitioned
CNN parameters at the gateway device. In \cite{1-4}, the memory and
communication requirements of edge devices are optimized considering
the limited computing resources of most edge computing devices. While
these efforts explore the potential of edge AI to significantly enhance
performance in various aspects for edge devices, they fail to acknowledge
the crucial constraints posed by limited energy and computing capabilities
of such devices, which are critical considerations in practical deployments
of edge AI, especially for UAVs.

Considering the fact that UAVs typically lack an external power supply,
energy harvesting is commonly employed in UAVs to facilitate long-term
sustained operation, e.g., radio-frequency energy harvesting and solar
energy harvesting \cite{UAV_MEC_1,UAV_MEC_2,UAV_MEC_3,UAV_MEC_4,arxiv1}.
In \cite{UAV_MEC_1}, a UAV-aided multiaccess edge computing system
with energy harvesting is studied, where the full-duplex protocol
is considered to realize simultaneously receiving confidential data
from the UAV and broadcasting the control instructions. A novel model
that uses a cluster of UAVs with energy harvesting capability as a
computational core is constructed in \cite{UAV_MEC_2}. It is capable
of providing long-term computational services for various scenarios.
In \cite{UAV_MEC_3}, a UAV-enabled mobile-edge computing (MEC) wireless-powered
system is studied. The computation rate maximization problems in a
UAV-enabled MEC wireless powered system are investigated under both
partial and binary computation offloading modes, subject to the energy-harvesting
causal constraint and the UAV's speed constraint. In \cite{UAV_MEC_4},
authors considered an online dynamic offloading and resource scheduling
algorithm to address the stochastic optimization problem of minimizing
energy and computing resource consumption of energy harvesting devices
while meeting the quality of service requirements of IoT devices.
However, while energy harvesting can extend the device's uptime, it
does not mitigate the energy consumption of edge AI. Insufficient
harvested energy still impacts edge AI operations, and the instability
in its collection significantly affects performance.

The split inference is an advanced edge AI architecture designed to
optimize energy consumption in edge devices \cite{3-1,3-2,3-3,3-4}.
Specifically, the overall model is partitioned into two components,
with one part being processed by the edge device and the other by
the server. It is important to note that the split inference approach
differs from the part offload model in edge computing, as the latter
partitions the data into two segments while still executing the entire
AI model on the edge device. In contrast to the part offload model,
split inference offers a significant reduction in computation and
communication overhead as well as ensuring data privacy. In \cite{3-5},
authors study the joint optimization of the model placement and online
model splitting decisions to minimize the energy-and-time cost of
device-edge co-inference in presence of wireless channel fading. A
multi-agent collaborative inference scenario, including a single edge
server and multiple user equipment (UEs), is studied in \cite{3-6}
to achieve fast and energy-efficient inference for all UEs. In \cite{3-7},
authors consider the scenario including multiple application tasks
with different options of deep learning models and different hyperparameter
settings. Although significant achievements have been made by combining
split inference with edge task offloading, the potential impact of
coupling among multiple sequential tasks is often overlooked. Moreover,
the adverse effects of random time-varying wireless channels can further
exacerbate the coupling above.

As previously discussed, in order to fully harness the advantages
of split inference and reduce the energy consumption, it is imperative
to address two key considerations before applying split inference
in UAVs. On the one hand, the energy level of UAVs and task delay
requirements must be taken into account when employing split inference.
On the other hand, the selection of transmission mode and transmit
power should be jointly optimized considering the randomness of task
arrival and time-varying wireless channel parameters. These two challenges
serve as the driving force behind this research.

\subsection{Contribution}

The energy-efficient power control for multi-task split inference
in UAVs is studied in this paper, taking into consideration multiple
sequential tasks. Upon the arrival of a task, it is initially determined
whether to employ split inference based on energy level of UAVs and
task delay constraints. Subsequently, the transmit power in each time
slot is optimized to minimize total power consumption while ensuring
the delay constraints of all tasks. Furthermore, we develop an optimal
algorithm to solve it. The main contributions of this paper can be
summarized as follows:
\begin{itemize}
\item An optimization programming (OP)-embedded tiny reinforcement learning
(TRL) framework is proposed. To be specific, we first present the
energy minimization as a two-timescale problem, where discrete and
continuous variables are segregated into two timescales. TRL-based
approach is utilized to determine whether split inference should be
employed for processing tasks based on the energy level of UAVs, length
of the transmission queue and task arrival rate. Furthermore, an OP
is embedded between the output and reward function of TRL to optimize
the transmit power in each time slot. Compared to conventional deep
reinforcement learning (DRL)-based algorithms, the proposed framework
can significantly decrease the size of action space and training time.
\item An optimal algorithm is developed to optimize the transmit power in
each time slot. In particular, the sample average approximation method
is applied to address the chance constraint caused by the random wireless
channel parameters. Then, the alternating direction method of multipliers
(ADMM) method is explored to transform the initial problem into multiple
subproblem to decrease the complexity. For each subproblem, the minimal
transmit power is achieved by optimizing the transmission time of
each task, which is motivated by our findings that increasing transmission
time leads to a monotonic decrease in energy consumption. The substitution
significantly reduce the feasible region and enables fast solution
according to the closed-form expression for optimal transmit power.
Simulation results show that the proposed OPETRL can achieve a higher
probability of successful data transmission with lower overall energy
consumption.
\end{itemize}
The rest of this paper is organized as follows. Section \ref{sec:System-model}
introduces the system model. The problem formulation and transformation
are given in Section \ref{sec:Problem-formulation}. The OP-embedded
TRL architecture and the optimal algorithm for transmit power are
given in Section \ref{sec:Tiny-learning-architecture}. Section \ref{sec:Conclusion }
shows the simulation results and Section \ref{sec:Conclusion } concludes
this paper.

\begin{figure}
\begin{centering}
\includegraphics[scale=0.65]{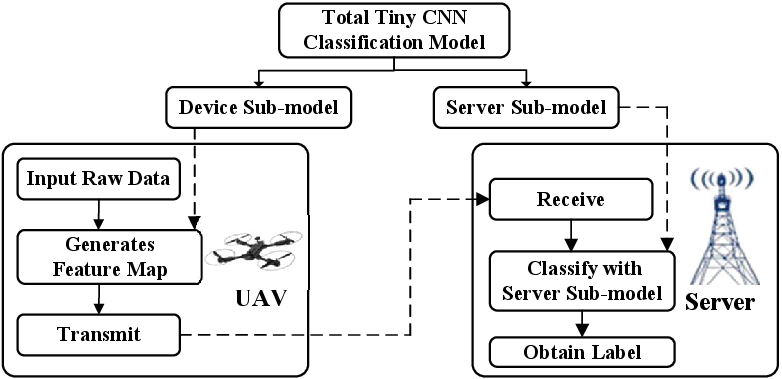}
\par\end{centering}
\caption{\label{fig:System}The split inference system in UAVs.}
\end{figure}

\section{System model\label{sec:System-model}}

\subsection{Network Model\label{subsec:Network-model}}

Consider the edge inference system shown in Fig. \ref{fig:System},
which consists of a UAV and a server. The UAV is equipped with some
sensors and collects the raw data, such as images for surveillance.
Then, the local data at UAV is compressed into features and sent to
the server to do a remote inference on a trained model. The considered
the UAV is powered by the harvested solar energy. The amount of harvested
solar energy is affected by clouds \cite{solar}, where the harvested
solar energy is reduced if there is a cloud between the sun and the
solar panel. The attenuation of the solar light passing through a
cloud can be modeled as \cite{solar}
\begin{equation}
\phi\left(d_{\text{cloud}}\right)=e^{-\beta d_{\text{cloud}}},
\end{equation}
where $\beta\geq0$ denotes the absorption coefficient that models
the optical characteristics of cloud and $d_{\text{cloud}}$ denotes
the cloud thickness. Therefore, the output power of the solar panel
used by the UAV can be modeled by the following function
\begin{equation}
P_{\text{s}}=\eta SGe^{-\beta d_{\text{cloud}}},\label{eq:solar power}
\end{equation}
where $\eta$ and $S$ are constants that denote the energy harvesting
efficiency and the equivalent area of the solar panels, respectively.
Moreover, $G$ denotes the average solar radiation intensity on earth.
Note that (\ref{eq:solar power}) is a simplified model that ignores
the change of $G$ over time during the day. The reasonableness of
this assumption lies in the fact that the duration of data transmission
is significantly shorter than the change period of $G$. Therefore,
in this paper, we consider the output power of the solar panel used
by the UAV is a constant.

\subsection{Transmission Model\label{subsec:Transmission-Model}}

A tiny CNN classification model is applied in this paper, which comprises
multiple convolutional (CONV) layers followed by multiple fully-connected
layers. We consider two patterns to complete classification tasks.
For the first pattern, the UAV transmit directly the raw data to the
server, called direct transmission (DT) pattern. The size of raw data
is set as $S$ bits. In this case, the total tiny CNN classification
model is deployed in the server. For the second pattern, the UAV generates
the feature map of raw data and sends the feature map to the server,
called computation-transmission (CT) pattern. In this case, the classification
model should be divided into two parts to implement split inference,
i.e., device sub-model and server sub-model. Following the widely
used designs \cite{3-3,3-1}, the splitting point is set after the
CONV layers, where the output of the last CONV layer is a feature
map with height $L_{\text{h}}$ and width $L_{\text{w}}$. Each element
in the feature map is quantized with a sufficiently high resolution
of $Q$ bits such that quantization errors are negligible. Note that
the completion time for classification tasks generally has a requirement,
although it may be less stringent in certain applications, such as
surveillance. We denote $C$ as the task completion time threshold,
which means that the duration from acquiring the raw data to obtaining
the classification result should be within $C$. The processing time
on the server is disregarded in this paper, as it is reasonable to
assume that the server possesses ample computing power and is not
constrained by any power limitations. Consequently, compared to the
processing time on the UAV and transmission delay, the processing
time on the server can be considered negligible. 

The UAV is allocated a narrow-band channel, whose bandwidth is denoted
by $W$. The time of the considered system is divided into multiple
time slots, whose set is denoted by $\mathcal{T}=\left\{ 1,...,T\right\} $.
In particular, the duration of each time slots is denoted by $\tau$,
which is smaller than the coherence time of the communication channel.
Thus, the power gain of wireless channel is constant within each time
slot. Denote $g_{t}$ as the small-scale fading channel coefficient
between the UAV and server at time slot $t$. In particular, $g_{t}$
is an independent and identically distributed (i.i.d.) random variable
over time slots, where the distribution function of $g\left(t\right)$
is set as $\mathcal{P}_{g}$. Moreover, we consider the zero-mean
Gaussian noise with variance $\sigma^{2}$. Therefore, the achievable
rate of the channel between the UAV and server at time slot $t$ can
be expressed as
\begin{equation}
R_{t}=W\log_{2}\left(1+\frac{p_{t}g_{t}}{\sigma^{2}d^{2}}\right),
\end{equation}
where $p_{t}$ and $d$ are the transmit power at time slot $t$ and
the distance from the UAV to server, respectively.

\subsection{Energy Model}

To capture the key features of the energy consumption during computation
and communication in the edge inference system, the energy consumption
model is first studied in this section. Specifically, we focus on
energy consumption during data transmission and processing in the
UAV, which are elaborated as follows.
\begin{itemize}
\item \textbf{Computing Energy:} For CT pattern, the UAV need to generate
the feature map of raw data. Denote $c$ as the number of floating-point
operations required to generate the feature map, which is depended
on the architecture of CNN deployed in the UAV. Moreover, denote $n_{\text{t}}$
as the number of floating-point operations required by all CONV layers.
The computing power of UAV at time slot $t$ is modeled as $p_{t}^{\text{C}}=kf^{3}$
as in \cite{frequency}, where $f$ and $k$ are the UAV's computational
speed and a coefficient depending on chip architecture, respectively.
Denote $f_{\max}$ as the maximum computational speed of UAV. The
number of floating-point operations that the UAV can handle in a cycle
is modeled as $n_{\text{d}}=N_{\text{c}}n_{\text{v}}/n_{\text{s}}$,
where $N_{\text{c}}$ and $n_{\text{s}}$ are the number of CPU cores
and the bit size of operating system, respectively . In addition,
$n_{v}$ is the bits of vector operation in each cycle. Therefore,
the execution time of generating the feature map is $t_{f}^{\text{C}}=\frac{n_{\text{t}}n_{\text{s}}}{fN_{\text{c}}n_{\text{v}}}$.
The energy consumption of generating the feature map is given by 
\begin{equation}
E^{\text{C}}=\frac{kf^{2}n_{\text{t}}n_{\text{s}}}{N_{\text{c}}n_{\text{v}}}.
\end{equation}
\item \textbf{Transmission Energy: }Denote $t_{i}^{\text{S}}$ and $t_{i}^{\text{T}}$
as the start time of transmission and the time consumption of transmission
for the $i$-th raw data or its feature map, respectively. Thus, the
transmission energy consumption is $E^{\text{T}}\left(i\right)=\sum_{t=t_{i}^{\text{S}}}^{t_{i}^{\text{S}}+t_{i}^{\text{T}}}p_{t}^{\text{T}}.$
\end{itemize}
Besides the energy consumption, we also model the battery energy level
of UAV. Denote $E_{t}^{\text{B}}$ as the battery energy level of
UAV at the start of time slot $t$. Thus, $E_{0}^{\text{B}}$ is the
initial energy level of UAV. During each time slot $t$, the UAV receives
a finite amount of energy, i.e., $\tau P_{\text{s}}$, from the environment.
Therefore, the energy in the battery at the start of time slot $t+1$
epoch is given by 
\begin{equation}
E_{t+1}^{\text{B}}=E_{t}^{\text{B}}+\tau P_{\text{s}}-E_{t}^{\text{T}}-E_{t}^{\text{C}},\forall t
\end{equation}
where $E_{t}^{\text{T}}$ and $E_{t}^{\text{C}}$ are the transmission
energy consumption and the computing energy consumption at time slot
$t$, respectively. 

\subsection{Traffic Model}

We assume that the sensor will generate the raw data at each time
slot with probability $q$. Denote $a_{t}=1$ as the raw data is generated
at time slot $t$, while the time slot that the $i$-th raw data arrives
is denoted by $t_{i}^{\text{A}}$. As mentioned in Section \ref{subsec:Transmission-Model},
the UAV will choose one pattern to process the raw data after it is
generated. For DT pattern, the UAV directly puts the raw data into
the transmission queue and process the next raw data. However, for
CT pattern, the UAV first gets the feature map and put it into the
transmission queue and then, processes the next raw data. Let the
computational variable $b_{t}\in\left\{ 0,1\right\} $ indicate that
whether the UAV process the raw data generated at time slot $t$ with
DT pattern. If yes, $b_{t}=1$; else $b_{t}=0$. Furthermore, the
UAV will choose a computational speed to process the raw data if $b_{t}=1$.
To guarantee the task completion time threshold $C$, the transmit
power $p_{t}$ should satisfy the following constraint 
\begin{equation}
\tau\stackrel[t=0]{t_{i}^{\text{A}}+C}{\sum}R_{t}\geq\stackrel[t=0]{t_{i}^{\text{A}}}{\sum}a_{t}\left[b_{t}L_{\text{h}}L_{\text{w}}Q+\left(1-b_{t}\right)S\right],\forall i.\label{eq:time threshold}
\end{equation}
Meanwhile, the transmit power $p_{t}$ should satisfy the following
causal constraint
\begin{equation}
\tau\stackrel[t=0]{t_{i}^{\text{S}}+t_{i}^{\text{T}}}{\sum}R_{t}\leq\stackrel[t=0]{t_{i}^{\text{A}}}{\sum}a_{t}\left[b_{t}L_{\text{h}}L_{\text{w}}Q+\left(1-b_{t}\right)S\right],\forall i.\label{eq:causal constraint}
\end{equation}

\section{Problem formulation \label{sec:Problem-formulation}}

In this section, we formulate the optimization problem to minimize
the energy consumption of UAV and then, transform it into a multistep
decision problem. 

\subsection{Problem formulation \label{subsec:Problem-formulation}}

We first formulate a problem to minimize the energy consumption of
UAV as follows
\begin{align}
\left(\text{P0}\right)\underset{\mathbf{p},\boldsymbol{\text{b}},\mathbf{f}}{\min} & \underset{t}{\sum}\left(E_{t}^{\text{T}}+E_{t}^{\text{C}}\right)\nonumber \\
s.t.\  & 0\leq E_{t}^{\text{B}}\leq E_{\text{\ensuremath{\max}}}^{\text{B}},\forall t,\label{eq:energy constraint}\\
 & 0\leq f_{i}\leq f_{\text{\ensuremath{\max}}},\forall i,\label{eq:frequency constraint}\\
 & b_{t}\in\left\{ 0,1\right\} ,\forall t,\label{eq:0-1 constraint}\\
 & b_{t}\leq a_{t},\forall t,\label{eq:task choose constraint}\\
 & 0\leq p_{t}\leq p_{\max},\forall t,\label{eq:power constraint}\\
 & \left(\ref{eq:causal constraint}\right),\left(\ref{eq:time threshold}\right).\nonumber 
\end{align}
where $\mathbf{p}=\left\{ p_{1},...,p_{t}\right\} $ and $\boldsymbol{\text{b}}=\left\{ b_{1},...,b_{t}\right\} $
are the sets of transmit power and computational variables. In addition,
$\mathbf{f}=\left\{ f_{1},...,f_{i}\right\} $ is the set of UAV computational
speed for different tasks. In problem (P0), constraint (\ref{eq:energy constraint})
means that the energy in the battery should be positive and smaller
than the maximum battery capacity, denoted by $E_{\text{\ensuremath{\max}}}^{\text{B}}$.
Constraint (\ref{eq:time threshold}) indicates that each task $i$
should be completed within the time threshold $C$. Moreover, constraint
(\ref{eq:causal constraint}) is a causal constraint where the transmission
process of the $i$-th raw data cannot be earlier than its generation
time. Constraints (\ref{eq:time threshold}) and (\ref{eq:causal constraint})
make problem (P0) quite different from other single-task optimization
problems, where the summation upper limits in the two constraints
are dependent on the previous resource allocation, thereby making
the two constraints a form of uncertainty. In particular, denote $t_{i}^{\text{QC}}$
and $t_{i}^{\text{QT}}$ as the queuing delay of the $i$-th raw data
in the transmission queue and computation queue, respectively. Thus,
the queuing delay of the transmission queue can be written as
\begin{equation}
t_{i}^{\text{QC}}=\max\left\{ 0,t_{i-1}^{\text{A}}+t_{i-1}^{\text{QC}}+a_{t}\left(1-b_{t}\right)t_{f_{i-1}}^{\text{C}}-t_{i}^{\text{A}}\right\} ,
\end{equation}
where $t_{0}^{\text{QC}}=0$. Moreover, the queuing delay of the transmission
queue is
\begin{equation}
t_{i}^{\text{QT}}=\max\left\{ 0,t_{i}^{\text{S}}-\left(t_{i}^{\text{A}}+t_{i}^{\text{QC}}+t_{f_{i}}^{\text{C}}\right)\right\} ,
\end{equation}
where $t_{i}^{\text{S}}=t_{i-1}^{\text{A}}+t_{i-1}^{\text{QC}}+t_{i-1}^{\text{QT}}+t_{i-1}^{\text{T}}$.

\subsection{Problem transformation \label{subsec:Problem-transformation}}

As mentioned above, problem (P0) differs from one-task optimization
problems \cite{frequency}, where the preceding tasks have impact
on the pattern choice of subsequent tasks and resource allocation
strategies in the subsequent time slots. Meanwhile, problem (P0) is
also different from the multi-task optimization problems, which do
not consider the task completion time threshold and queuing delay.
To solve problem (P0), we transform it into a series of multistep
decision problems, where the problem in the $t^{\prime}$-th time
slot is shown as 
\begin{align}
\left(\text{P1}\right)\underset{\mathbf{p}_{t^{\prime}},b_{t^{\prime}},f}{\min} & \underset{t\in\underline{\mathcal{T}}_{t^{\prime}}}{\sum}\left(E_{t}^{\text{T}}+E_{t}^{\text{C}}\right)\nonumber \\
s.t.\  & b_{t^{\prime}}\in\left\{ 0,1\right\} ,\\
 & b_{t^{\prime}}\leq a_{t^{\prime}},\\
 & 0\leq E_{t}^{\text{B}}\leq E_{\text{\ensuremath{\max}}}^{\text{B}},\forall t\in\underline{\mathcal{T}}_{t^{\prime}},\label{eq:P1-C4}\\
 & 0\leq p_{t}\leq p_{\max},\forall t\in\underline{\mathcal{T}}_{t^{\prime}},\label{eq:P1-C5}\\
 & \tau\left(\stackrel[t=t_{i}^{\text{S}}]{t^{\prime}-1}{\sum}\tilde{R}_{t}+\stackrel[t=t^{\prime}]{t_{i}^{\text{S}}+t_{i}^{\text{T}}}{\sum}R_{t}\right)\nonumber \\
 & \leq\stackrel[t=t^{\prime}]{t_{i}^{\text{A}}}{\sum}b_{t}L_{\text{h}}L_{\text{w}}Q+\left(1-b_{t}\right)S,\forall i\in\mathcal{I}_{t^{\prime}}.\label{eq:P1-C6}\\
 & \tau\left(\stackrel[t=t_{i}^{\text{S}}]{t^{\prime}-1}{\sum}\tilde{R}_{t}+\stackrel[t=t^{\prime}]{t_{i}^{\text{A}}+C}{\sum}R_{t}\right)\nonumber \\
 & \geq\stackrel[t=t^{\prime}]{t_{i}^{\text{A}}}{\sum}b_{t}L_{\text{h}}L_{\text{w}}Q+\left(1-b_{t}\right)S,\forall i\in\mathcal{I}_{t^{\prime}}\label{eq:P1-C7}
\end{align}
In problem (P1), it is defined that  $\mathcal{I}_{t^{\prime}}=\left\{ i|t_{i}^{\text{A}}\leq t^{\prime}\right\} $,
$\underline{\mathcal{T}}_{t^{\prime}}=\left\{ t^{\prime},...,t_{i}^{\text{A}}+C|i=\underset{i}{\arg\max}\left\{ t_{i}^{\text{A}}\leq t^{\prime}\right\} \right\} $
and $\mathbf{p}_{t^{\prime}}=\left\{ p_{t}^{\text{T}}|t\in\underline{\mathcal{T}}_{t^{\prime}}\right\} $.
Moreover, we set $f=0$ if $b_{t^{\prime}}=0$. Note that $\tilde{R}_{t}$
are known in the $t^{\prime}$-th time slot, which are dependent on
the process strategy and power allocation before the $t^{\prime}$-th
time slot.

While the initial problem (P0) can be transformed into a series of
multistep decision problems, problem (P1) is still intractable. In
the next section, an algorithm is designed for the multistep decision
problems.

\section{Tiny learning architecture-based optimal algorithm\label{sec:Tiny-learning-architecture}}

In practice, problem (P1) should be solved within a short time, while
the  UAV generally dose not have enough computation capacity to finish
it within the predefined time. DRL is a promising method to achieve
fast strategy generation. However, problem (P1) is a mixed integer
programming (MIP). While MIP can be solved by DRL, a complex neural
network (NN) with massive parameters is necessary. It is impractical
to deploy a complex NN in the UAV due to the limited energy. Hence,
in this section, we propose a tiny reinforcement learning (TRL) architecture,
which can significantly reduce the number of parameters of NN via
embedding the optimization programming into TRL. Furthermore, the
complex of optimization programming is also greatly decreased by analyzing
the non-linear objective function.

\subsection{TRL-based architecture}

As discussed above, problem (P1) is a MIP problem. It is unsuitable
for  UAVs to apply DRL to solve the MIP problem due to the massive
parameters of NN. Recalling and carefully analyzing problem (P1),
we find that it is not necessary to include the continuous variables
into the action space of the NN in reinforcement learning. Specifically,
the discrete variables in problem (P1) determine the strategies for
transmission modes in the sequential time slots, while the continuous
variables determine the allocation of transmission power in each timeslot.
As is widely acknowledged, reinforcement learning excels at addressing
continuous decision problems due to its inherent MDP nature. Therefore,
if the optimal solution for transmission power allocation in each
time slot can be achieved under the given strategies for completing
task patterns, it would enable the separation of continuous and discrete
variables in problem (P1). In this case, employing only reinforcement
learning to determine the discrete variables for transmission mode
selections within the sequential time slots would significantly simplify
the NN structure utilized in reinforcement learning. Additionally,
optimizing the continuous variables of transmission power allocation
can circumvent MIP, thereby substantially reducing computational time.
The proposed TRL architecture, called OP embedded TRL (OPETRL), is
given in Fig. \ref{fig:4.1 TRL-based architecture}. 
\begin{figure}
\begin{centering}
\includegraphics[scale=0.5]{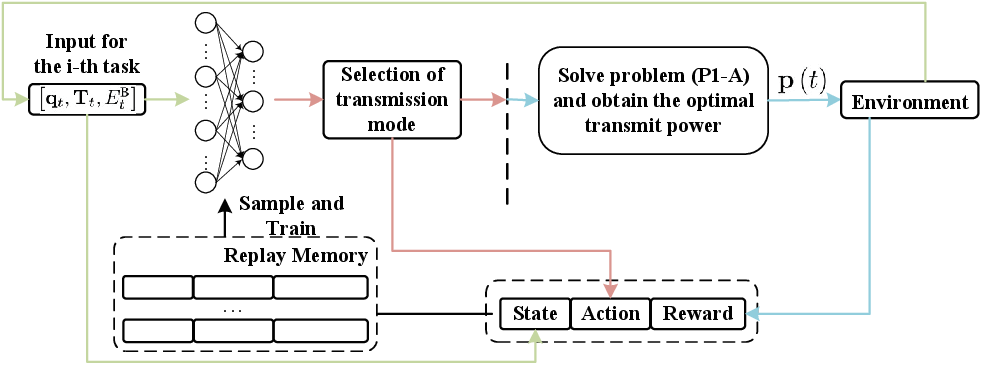}
\par\end{centering}
\caption{\label{fig:4.1 TRL-based architecture}The TRL-based architecture.}
\end{figure}

The problem (P1-A) in Fig. \ref{fig:4.1 TRL-based architecture} is
given as follows

\begin{align*}
\left(\text{P1-A}\right)\underset{\mathbf{p}_{t^{\prime}}}{\min} & \underset{t\in\underline{\mathcal{T}}_{t^{\prime}}}{\sum}\left(E_{t}^{\text{T}}+E_{t}^{\text{C}}\right)\\
s.t.\  & \left(\ref{eq:P1-C4}\right)-\left(\ref{eq:P1-C7}\right).
\end{align*}
In problem (P1-A), the computational variable $b_{t}$ is fixed, which
distinguishes it different from problem (P1). The key ideas of OPETRL
include two aspects. On the one hand, the solution processes of continuous
and discrete variables are segregated, with OP and TRL employed for
their respective resolutions. This segregation is critical as the
presence of mixed-integer variables can significantly reduce the convergence
rate of DRL and increase computational time required by OP. On the
other hand, the reward is not obtained by directly implementing action
in the environment. The action of TRL is considered as the input of
OP, whose output determines the reward of TRL in conjunction with
the action. In the following, we will provide a detailed description
of both TRL and OP.

\subsection{TRL for transmission mode selection}

As previously mentioned, TRL is employed to determine strategy for
selecting transmission modes. In this paper, we utilize DDQN, providing
a detailed exposition of the state space, action space, reward function,
and network architecture in subsequent sections.

\subsubsection{State space}

The state in time slot $t$ is denoted by $\mathbf{s}_{t}=\left[\mathbf{q}_{t},\mathbf{T}_{t},E_{t}^{\text{B}}\right]$,
where $\mathbf{q}_{t}$ represents the vector of the data size of
each task awaiting transmission and $\mathbf{T}_{t}$ denotes the
vector of remaining time allowed to complete each task. In each time
slot, the NN in DRL selects a transmission mode based on the transmission
queue, time threshold and the energy level upon task arrives.

\subsubsection{Action space}

Observing the state, the NN in DRL outputs the decision of transmission
mode at each time slot. The action for the $i$-th task is defined
as $a\left(t\right)$, where the size of action space equals $2$.
The action will be regarded as the input of problem (P1-A), whose
output further determines the reward of agent.

\subsubsection{Reward function}

The reward for the $i$-th task, denoted by $r\left(i\right)$, is
contingent upon the violation of constraints during a task. Specifically,
the action violating the completion time constraint will be penalized
with a negative reward, while adherence to it will result in a positive
reward. Note that the action choosing transmission modes varies in
a long-timescale manner, while the transmit power varies in a short-timescale
manner. Hence, the reward will not be generated at each time slot.
The reward for the $i$-th task is set as the energy consumption in
duration $\left[t_{i}^{\text{A}},t_{i}^{\text{S}}+t_{i}^{\text{T}}\right]$.

\subsubsection{Network Architecture}

In this paper, DDQN is used to determine the transmission mode for
each task \cite{DDQN}. There are one fully-connected hidden layers
with 32 neurons in the Q network. The ReLU is used as an activation
function in hidden layer. The parameter of the Q network is initialized
by a Xavier initialization scheme \cite{kumar2017weight}. Besides,
the target Q network is created by copying the corresponding parameters
of the Q network in the initialization phase.

The replay memory is updated with a new training sample $\left(\mathbf{s}\left(i\right),a\left(i\right),r\left(i\right),\mathbf{s}\left(i+1\right)\right)$
when the $i$-th task is completed within the time threshold or reaches
the time threshold. While updating the network parameters, a random
minibatch of training samples are sampled uniformly from replay memory.
The replay memory size is 1000 and the minibatch size is 64. The Q
network is trained by using the gradient descent method to minimize
the loss function, which is defined as the mean square error of the
difference between the Q value given by Q network and the target value
given by the target Q network. Moreover, the parameters of the Q network
will be copied into the target Q network each 20 time steps.

\subsection{Optimal algorithm for transmit power}

After determining the transmission mode, problem (P1-A) should be
addressed in order to obtain a resource allocation strategy that satisfies
the time threshold of each task. However, the inherent randomness
of wireless channels, which results in random achievable transmit
rate, poses a challenge in meeting time threshold constraint (\ref{eq:P1-C6}).
To solve it, we reformulate time threshold constraint (\ref{eq:P1-C6})
as a chance constraint, which is given by 
\begin{align}
 & \text{Pr}\left\{ \tau\left(\stackrel[t=t_{i}^{\text{S}}]{t^{\prime}-1}{\sum}\tilde{R}_{t}+\stackrel[t=t^{\prime}]{t_{i}^{\text{A}}+C}{\sum}R_{t}\right)\right.\nonumber \\
 & \left.\geq\stackrel[t=t^{\prime}]{t_{i}^{\text{A}}}{\sum}a_{t}\left[b_{t}L_{\text{h}}L_{\text{w}}Q+\left(1-b_{t}\right)S\right]\right\} \geq1-\epsilon,\forall i\in\mathcal{I}_{t^{\prime}}\label{eq:chance constraint}
\end{align}
In (\ref{eq:chance constraint}), $0<\epsilon\ll1$ is the maximum
probability that the time threshold constraint cannot be satisfied.
Moreover, the objective function in problem (P1-A) is also reformulated
as
\begin{equation}
\mathbb{E}_{\mathbf{g}\sim\mathcal{P}_{g}}\left[\underset{t\in\underline{\mathcal{T}}_{t^{\prime}}}{\sum}E_{t}^{\text{T}}\right].
\end{equation}
Hence, problem (P1-A) can be rewritten as

\begin{align*}
\left(\text{P1-B}\right)\underset{\mathbf{p}_{t^{\prime}}}{\min} & \mathbb{E}_{\mathbf{g}\sim\mathcal{P}_{g}}\left[\underset{t\in\underline{\mathcal{T}}_{t^{\prime}}}{\sum}E_{t}^{\text{T}}\right]\\
s.t.\  & \left(\ref{eq:P1-C4}\right)-\left(\ref{eq:P1-C6}\right),\left(\ref{eq:chance constraint}\right).
\end{align*}

\subsubsection{Problem transformation}

Problems (P1-B) is difficult to solve exactly since it is difficult
to obtain the closed-form expressions of (\ref{eq:chance constraint}).
To solve it, we approximate the computationally intractable chance
constraints (\ref{eq:chance constraint}) via the sample average approximation
(SAA) technique \cite{so2013distributionally,yang2021multicast}.
The SAA technique offers two advantages in approximating chance constraints.
On the one hand, it is applicable to general channel distribution
models, which is crucial as the practical distribution of channel
coefficients often defies typical modeling approaches. On the other
hand, by replacing the chance constraint with a finite set of simpler
constraints under the SAA technique, we can ensure that the performance
gap between solutions to the sample approximation problem and those
to the original problem remains minimal. Denote $\mathbf{g}^{k}$
as the $k$-th independent sample of $\mathbf{g}$ generated according
to $\mathcal{P}_{g}$, where $k\in\mathcal{K}$ with $\mathcal{K}\overset{\triangle}{=}\left\{ 1,...,K\right\} $.
According to the law of large numbers, if the sample size $K$ is
sufficiently enough, we can approximately substitute the objective
functions in problems (P1-B) with
\begin{align}
\mathbb{E}_{\mathbf{g}\sim\mathcal{P}_{g}}\left[\underset{t\in\underline{\mathcal{T}}_{t^{\prime}}}{\sum}\left(E_{t}^{\text{T}}+E_{t}^{\text{C}}\right)\right]\approx & \frac{1}{K}\underset{k\in\mathcal{K}}{\sum}\underset{t\in\underline{\mathcal{T}}_{t^{\prime}}}{\sum}\left(E_{t}^{\text{T},k}+E_{t}^{\text{C},k}\right).\label{eq:obj sampling}
\end{align}
Moreover, the chance constraint (\ref{eq:chance constraint}) can
be replaced with 
\begin{align}
 & \tau\left(\stackrel[t=t_{i}^{\text{S}}]{t^{\prime}-1}{\sum}\tilde{R}_{t}+\stackrel[t=t^{\prime}]{t_{i}^{\text{A}}+C}{\sum}R_{t}^{k}\right)\nonumber \\
 & \geq\stackrel[t=t^{\prime}]{t_{i}^{\text{A}}}{\sum}a_{t}\left[b_{t}L_{\text{h}}L_{\text{w}}Q+\left(1-b_{t}\right)S\right],\forall i\in\mathcal{I}_{t^{\prime}},k\in\mathcal{K}.\label{eq:chance constraint sampling}
\end{align}
It has been proved that if the number of samples $K$ is no less than
$K^{*}$ with 
\begin{align}
K^{*}= & \left\lceil \frac{1}{\epsilon}\left(N-1+\log\frac{1}{\theta}\sqrt{2\left(N-1\right)\log\frac{1}{\theta}+\log^{2}\frac{1}{\theta}}\right)\right\rceil \label{eq:number of samples}
\end{align}
for any $\theta\in\left(0,1\right)$, then any solution satisfying
constraint (\ref{eq:chance constraint sampling}) may satisfy chance
constraint (\ref{eq:chance constraint}) with a probability at least
$1-\theta$ \cite{so2013distributionally,yang2021multicast}. 

The number of variables and constraints in the transformed problem
is significantly larger (approximately K times that of problem (P0)),
particularly when dealing with a large number of samples $K$. Moreover,
it should be noted that the transformed problem represents a non-convex
optimization problem with constraint (\ref{eq:P1-C6}). However, due
to limited computational capacity and energy resources, the UAV is
unable to achieve a solution within a time slot. To handle it, we
reframe it as a global consensus problem \cite{boyd2011distributed},
which is prevalent in the domains of distributed computing and multi-agent
systems, necessitating all agents to converge on a shared value crucial
for computation. In the transformed problem, we focus on optimizing
the transmit power and computational speed under random channel coefficients,
where the optimization of the computational speed executes once after
task arrival. Therefore, we can consider each sample as an individual
agent, with the transmit power at the current time slot being treated
as a shared variable among all agents. Inspired by the approaches
in \cite{tang2019service}, a parallel architecture is proposed for
problem (P1-B). To be specific, we replace problem (P1-B) as $K$
subproblems, where the $k$-th subproblem is 
\begin{align}
\left(\text{P1-C}\right)\underset{\mathbf{p}_{t^{\prime}}^{k}}{\min} & \underset{t\in\underline{\mathcal{T}}_{t^{\prime}}}{\sum}E_{t}^{\text{T},k}\nonumber \\
s.t.\  & \left(\ref{eq:P1-C4}\right)-\left(\ref{eq:P1-C6}\right),\nonumber \\
 & \tau\left(\stackrel[t=t_{i}^{\text{S}}]{t^{\prime}-1}{\sum}\tilde{R}_{t}+\stackrel[t=t^{\prime}]{t_{i}^{\text{A}}+C}{\sum}R_{t}^{k}\right)\nonumber \\
 & \geq\stackrel[t=t^{\prime}]{t_{i}^{\text{A}}}{\sum}a_{t}\left[b_{t}L_{\text{h}}L_{\text{w}}Q+\left(1-b_{t}\right)S\right],\forall i\in\mathcal{I}_{t^{\prime}}.\label{eq:a sample}
\end{align}
Note that the variables in constraints (\ref{eq:P1-C4})-(\ref{eq:P1-C6})
within problem (P1-C) are the corresponding ones under the sample
$\mathbf{g}^{k}$. To obtain the transmit power, we can parallelly
address the $K$ subproblems and then recover $\mathbf{p}_{t^{\prime}}$
and $f$ from the solutions via ADMM. In the following, we present
the optimal algorithm for each subproblem.

\subsubsection{Optimal algorithm for subproblem}

We aim to propose an algorithm that can efficiently achieve the solution
for problem (P1-C) with low computational complexity. However, constraint
(\ref{eq:P1-C6}) is a non-convex constraint, which results in extra
computational complexity. Therefore, we present a problem formulation
for the transmission of the $i$-th task as follows
\begin{align}
\left(\text{P2}\right)\underset{\mathbf{p}_{t_{i}^{\text{S}}}^{k}}{\min} & \sum_{t=t_{i}^{\text{S}}}^{t_{i}^{\text{S}}+t_{i}^{\text{T,\ensuremath{\max}}}}p_{t}^{k}\nonumber \\
s.t.\  & \left(\ref{eq:P1-C5}\right),\nonumber \\
 & \tau\sum_{t=t_{i}^{\text{S}}}^{t_{i}^{\text{S}}+t_{i}^{\text{T}}}R_{t}^{k}\geq D_{i},\label{eq:single task rate}
\end{align}
where $D_{i}=b_{t_{i}^{\text{A}}}L_{\text{h}}L_{\text{w}}Q+\left(1-b_{t_{i}^{\text{A}}}\right)S$,
$\mathbf{p}_{t_{i}^{\text{S}}}^{k}=\left(p_{t_{i}^{\text{S}}}^{k},...,p_{t_{i}^{\text{S}}+t_{i}^{\text{T,\ensuremath{\max}}}}^{k}\right)$
and $t_{i}^{\text{T,\ensuremath{\max}}}$ is the allowed transmission
time of the $i$-th task. Note that in problem (P2), the time consumption
of transmission for the $i$-th task is given. By the Karush-Kuhn-Tucker
(KKT) conditions, the sufficient and necessary conditions for the
optimal solution of problem (P2), denoted by $\mathbf{p}_{t_{i}^{\text{S}}}^{k}$,
can be stated in the following theorem. \begin{theorem}\label{thm:optimal solution}The
optimal solution $\mathbf{p}_{t_{i}^{\text{S}}}^{k,*}$ of problem
(P2) satisfies the following sufficient and necessary conditions 
\begin{align}
p_{t}^{k,*} & =\min\left\{ \left[\tau We^{\frac{\frac{D_{i}}{\tau W}-\sum_{t=t_{i}^{\text{S}}}^{t_{i}^{\text{S}}+t_{i}^{\text{T,\ensuremath{\max}}}}\log\left(\tau Wh_{t}^{k}\right)}{t_{i}^{\text{T,\ensuremath{\max}}}}}-\frac{1}{h_{t}^{k}}\right]^{+},p_{\max}\right\} \label{eq:optimal power}
\end{align}
for all $t\in\mathcal{T}_{t_{i}^{\text{S}}}\overset{\triangle}{=}\left(t_{i}^{\text{S}},...,t_{i}^{\text{S}}+t_{i}^{\text{T,\ensuremath{\max}}}\right)$,
where $h_{t}^{k}=\frac{g_{t}^{k}}{\sigma^{2}d^{2}}$ and $\left[x\right]^{+}=\max\left\{ 0,x\right\} $.\end{theorem}\begin{proof}Please
refer to Appendix \ref{Proof:optimalsolution}.\end{proof}Moreover,
we proof that the total energy under the optimal solution decreases
monotonically with the transmission time in the following theorem.\begin{theorem}\label{thm:decrease with time}The
objective value of problem (P2) associated with the optimal solution
decreases monotonically with the transmission time.\end{theorem}\begin{proof}Please
refer to Appendix \ref{Proof:DecreasewithTime}.\end{proof}

\textit{Remark: According to Theorem \ref{thm:optimal solution} and
\ref{thm:decrease with time}, the optimal strategy for transmit power
can be obtained by optimizing the allowed transmission time of each
task. Moreover, the transmission process will last until the completion
time threshold of the last task in the current queue is reached.}

Hence, we formulate the optimization problem at the $t^{\prime}$-th
time slot as 
\begin{align}
\left(\text{P3}\right)\underset{\mathbf{t}_{i}^{\text{T,\ensuremath{\max}}}}{\min} & \underset{i\in\mathcal{I}_{t^{\prime}}}{\sum}E_{i}\left(t_{i}^{\text{T,\ensuremath{\max}}}\right)\nonumber \\
s.t.\  & t_{i}^{\text{S}}+t_{i}^{\text{T,\ensuremath{\max}}}\leq t_{i}^{\text{A}}+C,i\in\mathcal{I}_{t^{\prime}}\label{eq:P3-C1}\\
 & t_{l}^{\text{S}}+t_{l}^{\text{T,\ensuremath{\max}}}=t_{l}^{\text{A}}+C\label{eq:P3-C2}\\
 & t_{i}^{\text{S}}=t_{i-1}^{\text{S}}+t_{i-1}^{\text{T,\ensuremath{\max}}},i\in\mathcal{I}_{t^{\prime}},\label{eq:P3-C3}
\end{align}
where $E_{i}\left(t_{i}^{\text{T,\ensuremath{\max}}}\right)$ is the
total transmission energy for the $i$-th task. Moreover, in (\ref{eq:P3-C2}),
$l$ is the index of the last task in the set $\mathcal{I}_{t^{\prime}}$.
The optimal solution of problem (P3) should satisfy the condition
stated in Theorem \ref{thm:optimal time}. \begin{theorem}\label{thm:optimal time}If
the optimal allowed transmission time of tasks $m,n\in\mathcal{I}_{t^{\prime}}$
satisfies $t_{i}^{\text{S}}+t_{i}^{\text{T,\ensuremath{\max}}}\leq t_{i}^{\text{A}}+C,\forall i\in\left\{ m,n\right\} $,
then, it should also satisfy 
\begin{align}
 & \frac{t_{n}^{\text{T,\ensuremath{\max}}}}{t_{m}^{\text{T,\ensuremath{\max}}}}\nonumber \\
= & \frac{e^{\frac{D_{m}}{\tau Wt_{m}^{\text{T,\ensuremath{\max}}}}}e^{\frac{1}{t_{n}^{\text{T,\ensuremath{\max}}}}\sum_{t=t_{n}^{\text{S}}}^{t_{n}^{\text{S}}+t_{n}^{\text{T,\ensuremath{\max}}}}\log\left(\tau Wh_{t}^{k}\right)}}{e^{\frac{D_{n}}{\tau Wt_{n}^{\text{T,\ensuremath{\max}}}}}e^{\frac{1}{t_{m}^{\text{T,\ensuremath{\max}}}}\sum_{t=t_{m}^{\text{S}}}^{t_{m}^{\text{S}}+t_{m}^{\text{T,\ensuremath{\max}}}}\log\left(\tau Wh_{t}^{k}\right)}},\forall n,m\in\mathcal{I}_{t^{\prime}}.\label{eq:ratio of time}
\end{align}
\end{theorem}\begin{proof}Please refer to Appendix \ref{Proof:OptimalTime}.\end{proof}Moreover,
considering the constraints (\ref{eq:P3-C2}) and (\ref{eq:P3-C3}),
we proposed an optimal algorithm to solve problem (P1-C) as shown
in Algorithm \ref{algorithm:TranmitPower}.\begin{algorithm}[!t]
\caption{Optimal algorithm for tranmit power}  
\begin{algorithmic}[1] \label{algorithm:TranmitPower}

\STATE \textbf{Initialization}
\STATE $\bullet$ Input task set $\mathcal{I}_{t^{\prime}}$ and the transmission data of each task $D_{i},\forall i\in\mathcal{I}_{t^{\prime}}$. 
\STATE $\bullet$ Set the allowed transmission time of the first task as $t_{1}^{\text{T,\ensuremath{\max},(0)}}=t_{1}^{\text{A}}+C-\max\left\{ t_{1}^{\text{S}},t^{\prime}\right\} $. 
\STATE $\bullet$ Set the allowed transmission time of other tasks as $t_{i}^{\text{T,\ensuremath{\max},(0)}}=t_{i}^{\text{A}}+C-t_{i}^{\text{S}},\forall i\in\mathcal{I}_{t^{\prime}}\backslash\left\{ 1\right\} $.
\STATE $\bullet$ Set $l=1$.
\STATE
\STATE \textbf{while} $l<\mathit{Loop}$
\STATE Calculate the values on both sides of  (\ref{eq:ratio of time}), denoted by $f_{n,m}^{L}$ and $f_{n,m}^{R},\forall n,m\in\mathcal{I}_{t^{\prime}}$, respectively.
\STATE Set $\left\{ n_{\max},m_{\max}\right\} =\underset{n,m}{\arg\max}\left(f_{n,m}^{R}-f_{n,m}^{L}\right)$.
\STATE \textbf{if} $f_{n_{\max},m_{\max}}^{R}-f_{n_{\max},m_{\max}}^{L}>0$
\STATE Set $t_{m_{\max}}^{\text{T,\ensuremath{\max}},(l)}=  t_{m_{\max}}^{\text{T,\ensuremath{\max}},(l-1)}-1$ \textbf{if} $t_{m_{\max}+1}^{\text{S},(l-1)}>t_{m_{\max}+1}^{\text{A}}+t_{f}^{\text{C}}$.
\STATE Set $t_{n_{\max}}^{\text{T,\ensuremath{\max}},(l)}=t_{n_{\max}}^{\text{T,\ensuremath{\max}},(l-1)}+1$  \textbf{if} $t_{n_{\max}}^{\text{T,\ensuremath{\max}},(l-1)}<t_{n_{\max}}^{\text{A}}+C-t_{n_{\max}}^{\text{S},(l-1)}$.
\STATE \textbf{elseif} $f_{n_{\max},m_{\max}}^{R}-f_{n_{\max},m_{\max}}^{L}<0$
\STATE Set $t_{n_{\max}}^{\text{T,\ensuremath{\max}},(l)}=  t_{n_{\max}}^{\text{T,\ensuremath{\max}},(l-1)}-1$ \textbf{if} $t_{n_{\max}+1}^{\text{S},(l-1)}>t_{n_{\max}+1}^{\text{A}}+t_{f}^{\text{C}}$.
\STATE Set $t_{m_{\max}}^{\text{T,\ensuremath{\max}},(l)}=t_{m_{\max}}^{\text{T,\ensuremath{\max}},(l-1)}+1$  \textbf{if} $t_{m_{\max}}^{\text{T,\ensuremath{\max}},(l-1)}<t_{m_{\max}}^{\text{A}}+C-t_{m_{\max}}^{\text{S},(l-1)}$.
\STATE \textbf{else}  Break
\STATE \textbf{endif} 
\STATE Set $l=l+1$.
\STATE \textbf{end} 
\STATE Obtain the transmit power according to equation (\ref{eq:optimal power}).
\STATE
\RETURN Transmit power of each subproblem.
\end{algorithmic} 
\end{algorithm}

\subsubsection{Restoration of transmit power}

With Algorithm \ref{algorithm:TranmitPower}, we can obtain the copies
of transmit power at each time slot in all $K$ subproblems, i.e.,
$\mathbf{p}_{t^{\prime}}^{k},\forall k\in\mathcal{K}\overset{\triangle}{=}\left\{ 1,...,K\right\} $.
In the following, we restore transmit power at each time slot from
their copies, i.e., restoring $\mathbf{p}_{t^{\prime}}$ from $\mathbf{p}_{t^{\prime}}^{k},\forall k\in\mathcal{K}$.

The ADMM can be leveraged for variable restoration, with a key finding
being that it drives local copies of global variables towards their
average value \cite{boyd2011distributed,tang2019service,ADMM2}. Particular,
since the transmit power at each time slot is a continuous variable,
it can be restored by $\mathbf{p}_{t^{\prime}}\approx\frac{1}{K}\sum_{k\in\mathcal{K}}\mathbf{p}_{t^{\prime}}^{k}$.
Therefore, the total algorithm for problem (P1-B) is shown in Algorithm
\ref{algorithm:OptAlgorithm}.\begin{algorithm}[!t]
\caption{Optimal algorithm for (P1-B)}  
\begin{algorithmic}[1] \label{algorithm:OptAlgorithm}

\STATE \textbf{Initialization}
\STATE $\bullet$ Input the number of samples $K$.
\STATE
\STATE Generate $K$ subproblems based on $K$ samples of channel coefficients.
\STATE Obtain the solution of each subproblem by applying Algorithm \ref{algorithm:OptAlgorithm}.
\STATE Restore transmit power at each time slot from their copies.
 
\RETURN Solution of problem (P1-B).
\end{algorithmic} 
\end{algorithm}

\section{Simulation Results\label{sec:Simulation-Results}}

In this section, we provide numerical examples to demonstrate the
effectiveness of the proposed algorithm. We consider a UAV and a server,
where the distance between them is set as $d=100$m. The duration
of each time slot is set as $\tau=0.1\text{s}$. The receiver noise
power and the channel power at the reference distance are set as $\sigma^{2}$=
\textminus 110dBm and $\rho=\text{\textminus}60\text{dB}$, respectively.
Moreover, the bandwidth is set as $W=2$MHz.

We compare the proposed OPETRL with the following two offloading policies. 
\begin{itemize}
\item One-task policy: The policy optimizes the computation frequency and
transmit power for each task independently, without taking into account
the coupling among multiple sequential tasks. 
\item Greedy: The UAV chooses transmission mode and transmit power according
to the current wireless channel parameters. 
\end{itemize}
\begin{figure}
\begin{centering}
\subfloat[\label{fig:P vs QoS}]{\begin{centering}
\includegraphics[scale=0.45]{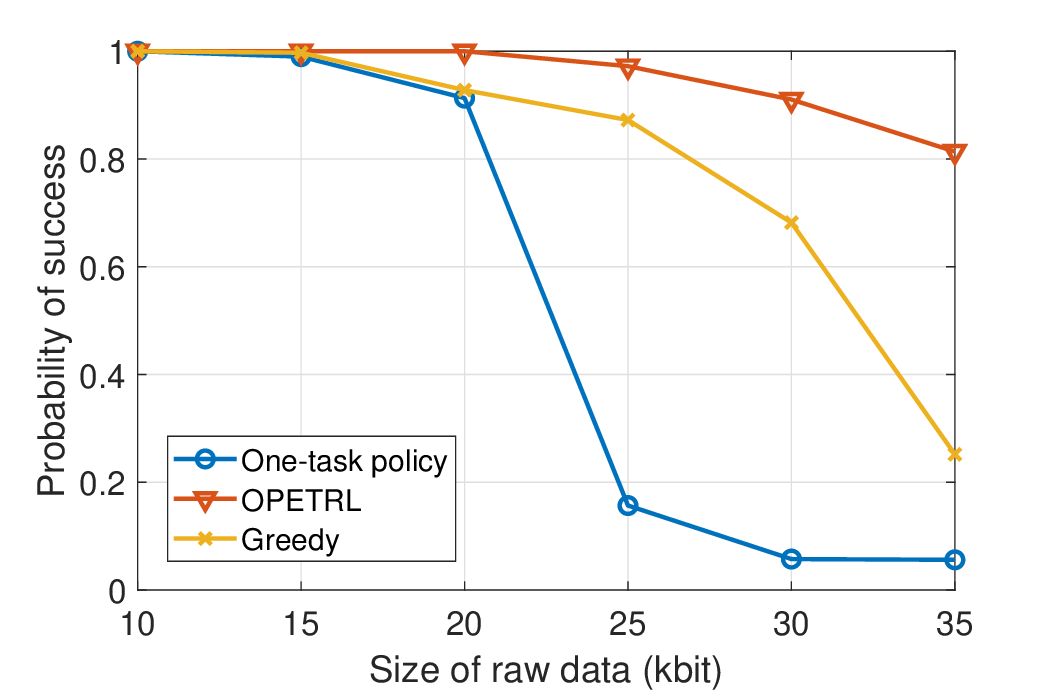}
\par\end{centering}
}
\par\end{centering}
\begin{centering}
\subfloat[\label{fig:E vs QoS}]{\begin{centering}
\includegraphics[scale=0.45]{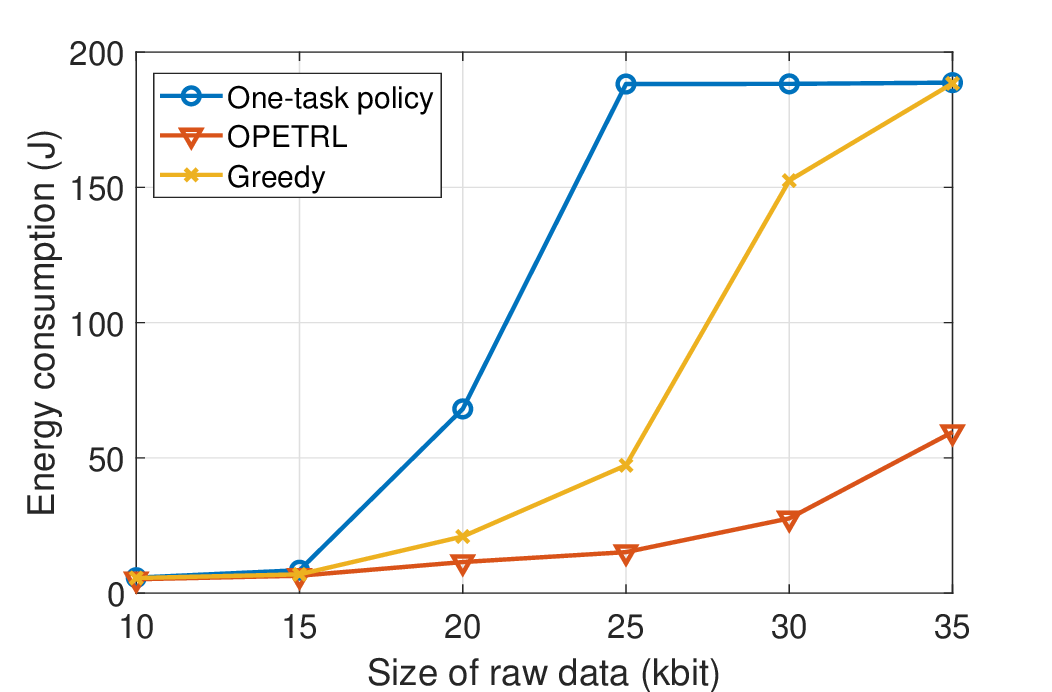}
\par\end{centering}
}
\par\end{centering}
\caption{\label{fig: QoS}The effects of the size of raw data $S$. (a) The
probability of successful data transmission versus $S$. (b) The overall
energy consumption versus $S.$}
\end{figure}
To reveal the effects of the size of raw data, we plot the probability
of successful data transmission and the overall energy consumption
versus the size of raw data in Fig. \ref{fig: QoS}. From the numerical
results, it can be obtained that the OPETRL generally outperforms
the baselines. Moreover, both the proposed OPETRL and Greedy policy
exhibit a slower decline in performance compared to that of the one-task
policy.

For Fig. \ref{fig:P vs QoS}, the increased size of the raw data results
in an extended transmission time, thereby leading to a cascade of
transmission failures for the later tasks when multiple tasks arrive
intensively under the one-task policy. Thus, the probability of successful
data transmission will fast decrease when the size of raw data is
larger than 20 kbit. Moreover, the most of data transmission fails
when $S\geq25$ kbit under the one-task policy.

Similarly, as shown in Fig. \ref{fig:E vs QoS}, the increased size
of the raw data leads to a corresponding increase in overall energy
consumption. This is because that the cascade of transmission failures
results in the UAV transmitting data with maximum power most of the
time. Meanwhile, the UAV will choose the CT pattern for most task,
thereby consuming computational energy.
\begin{center}
\begin{figure}
\begin{centering}
\subfloat[\label{fig:P vs lambda}]{\begin{centering}
\includegraphics[scale=0.45]{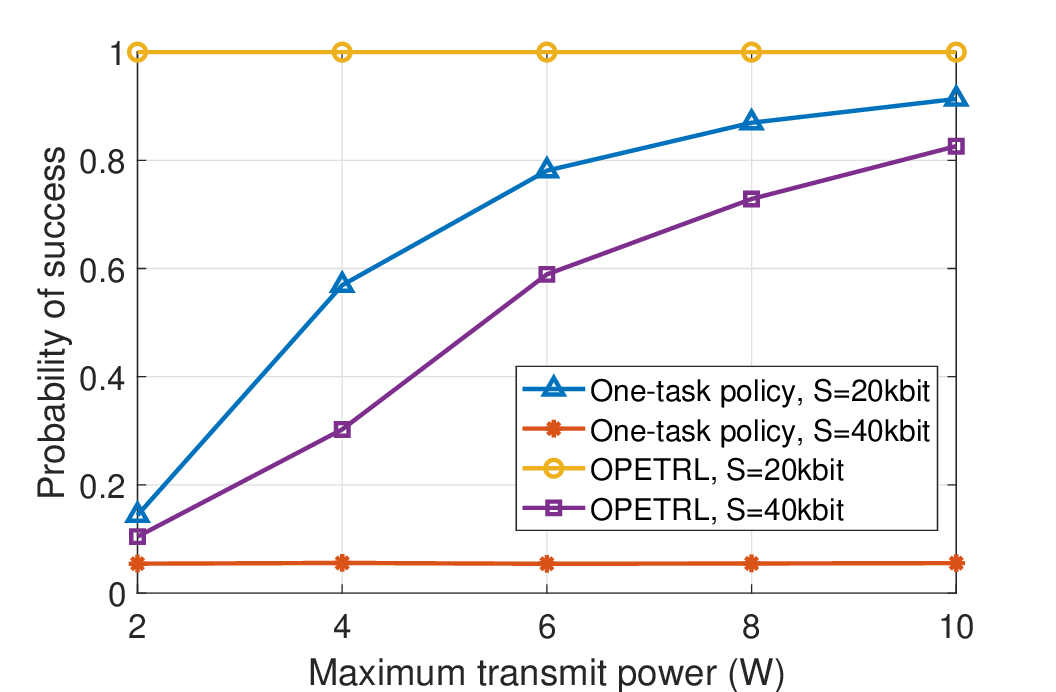}
\par\end{centering}
}
\par\end{centering}
\begin{centering}
\subfloat[\label{fig:E vs lambda}]{\begin{centering}
\includegraphics[scale=0.45]{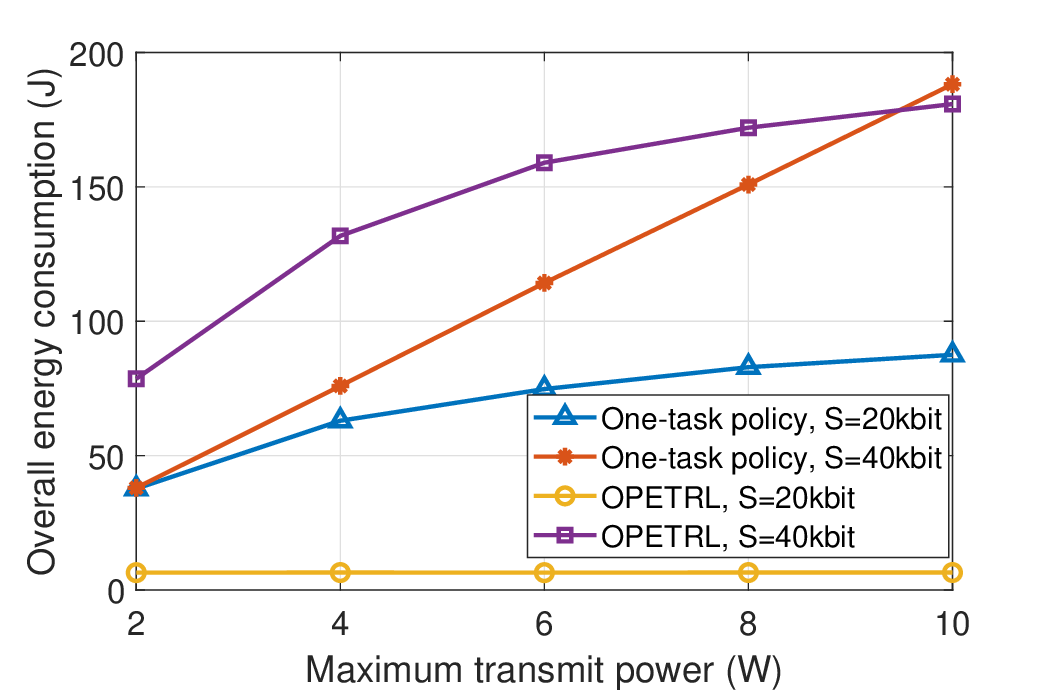}
\par\end{centering}
}
\par\end{centering}
\caption{\label{fig:lambda}The effects of the maximum transmit power $P_{\max}$.
(a) The probability of successful task completion versus $P_{\max}$.
(b) The overall energy consumption versus $P_{\max}$.}
\end{figure}
\par\end{center}

Furthermore, the impact of the maximum transmit power $P_{\max}$
is depicted in Fig. \ref{fig:lambda}. As we can see, when the size
of raw data is small, the transmit power cannot significantly effect
the probability of successful data transmission and the overall energy
consumption under both the proposed OPETRL and the Greedy policy.
This is attributed to the fact that the transmit power in both aforementioned
policies is smaller than its maximum value. However, the UAV transmits
data with maximum power most of the time when employing the one-task
policy, a decrease in $P_{\max}$ leads to a reduction in both the
probability of successful data transmission and the overall energy
consumption.

As shown in Fig. \ref{fig:E vs lambda}, under one-task policy, a
significant number of time slots remain idle while the maximum power
is utilized for transmission. Hence, as the maximum power increases,
there is a proportional increase in total energy consumption without
any corresponding enhancement in the probability of successful task
completion. However, OPETRL maximizes the utilization of time slots
while ensuring that the transmit power in each slot remains below
its maximum threshold. Therefore, an increase in the maximum transmit
power results in a nearly logarithmical rise in total energy consumption.

\section{Conclusion\label{sec:Conclusion }}

In this paper, a TRL-based architecture for multiple-task split inference
in IoT networks is studied. The problem of minimizing overall energy
consumption is initially formulated as a two-timescale optimization
problem. Then, the problem is decomposed into two parts, where TRL
is utilized to determine the transmission mode for each task and OP
is used to optimize the transmit power in each time slot. Furthermore,
we observe that the optimal transmit power decreases monotonically
with increasing transmission time. Thus, the minimal transmit power
is achieved by optimizing the transmission time of each task. Simulation
results show that the proposed OPETRL can achieve a higher probability
of successful data transmission with lower overall energy consumption.

\begin{appendix}

\subsection{Proof of Theorem \ref{thm:optimal solution}\label{Proof:optimalsolution}}

By constructing the Lagrangian function of problem (P2), it can be
obtained 
\begin{align}
\mathcal{L}\left(\mathbf{p}_{t_{i}^{\text{S}}}^{k},\mathbf{\mu},\mathbf{\lambda},\mathbf{\xi},\eta\right)= & \sum_{t=t_{i}^{\text{S}}}^{t_{i}^{\text{S}}+t_{i}^{\text{T,\ensuremath{\max}}}}\left(p_{t}^{k}-\mu_{t}p_{t}^{k}+\lambda_{t}\left(p_{t}^{k}-p_{\max}\right)\right)\nonumber \\
 & -\eta\left(\tau\left(\sum_{t=t_{i}^{\text{S}}}^{t_{i}^{\text{S}}+t_{i}^{\text{T,\ensuremath{\max}}}}R_{t}^{k}\right)+D\right),
\end{align}
where $\mu_{t}$ and $\lambda_{t}$ are the Lagrange multipliers associated
with constraint (\ref{eq:P1-C5}), $\eta$ is the Lagrange multipliers
associated with constraint (\ref{eq:single task rate}). Moreover,
we define $\mathbf{\mu}=\left(\mu_{t_{i}^{\text{S}}},...,\mu_{t_{i}^{\text{S}}+t_{i}^{\text{T,\ensuremath{\max}}}}\right)$
and $\mathbf{\lambda}=\left(\lambda_{t_{i}^{\text{S}}},...,\lambda_{t_{i}^{\text{S}}+t_{i}^{\text{T,\ensuremath{\max}}}}\right)$.
Therefore, the partial derivative of $\mathcal{L}\left(\mathbf{p}_{t_{i}^{\text{S}}}^{k},\mathbf{\mu},\mathbf{\lambda},\mathbf{\xi},\eta\right)$
with respect to $p_{t}^{k}$ is given by
\begin{equation}
\frac{\partial\mathcal{L}\left(\mathbf{p}_{t_{i}^{\text{S}}}^{k},\mathbf{\mu},\mathbf{\lambda},\mathbf{\xi},\eta\right)}{\partial p_{t}^{k}}=1-\mu_{t}+\lambda_{t}-\eta\frac{\tau Wh_{t}^{k}}{1+h_{t}^{k}p_{t}^{k}}.\label{eq:L function}
\end{equation}
Then, the KKT conditions are written as follows
\begin{equation}
\frac{\partial\mathcal{L}\left(\mathbf{p}_{t_{i}^{\text{S}}}^{k,*},\mathbf{\mu},\mathbf{\lambda},\mathbf{\xi},\eta\right)}{\partial p_{t}^{k,*}}=0,\forall t\in\mathcal{T}_{t_{i}^{\text{S}}}\label{eq:L function to p}
\end{equation}
\begin{equation}
\mu_{t}p_{t}^{k,*}=0,\lambda_{t}\left(p_{t}^{k,*}-p_{\max}\right)=0,\forall t\in\mathcal{T}_{t_{i}^{\text{S}}}
\end{equation}
\begin{equation}
\eta\left(\tau\left(\sum_{t=t_{i}^{\text{S}}}^{t_{i}^{\text{S}}+t_{i}^{\text{T,\ensuremath{\max}}}}R_{t}^{k}\left(\mathbf{p}_{t_{i}^{\text{S}}}^{k,*}\right)\right)+D\right)=0\label{eq:L function to QoS}
\end{equation}
\begin{equation}
0\leq p_{t}^{k,*}\leq p_{\max},\forall t\in\mathcal{T}_{t_{i}^{\text{S}}}
\end{equation}
\begin{equation}
\tau\sum_{t=t_{i}^{\text{S}}}^{t_{i}^{\text{S}}+t_{i}^{\text{T,\ensuremath{\max}}}}R_{t}^{k}\left(\mathbf{p}_{t_{i}^{\text{S}}}^{k,*}\right)\geq D\label{eq:QoS for optimal solution}
\end{equation}
\begin{equation}
\mu_{t}\geq0,\lambda_{t}\geq0,\forall t\in\mathcal{T}_{t_{i}^{\text{S}}}.
\end{equation}
According to (\ref{eq:L function}) and (\ref{eq:L function to p}),
we have 
\begin{equation}
\eta=\left(1-\mu_{t}+\lambda_{t}\right)\frac{1+h_{t}^{k}p_{t}^{k,*}}{\tau Wh_{t}^{k}}.
\end{equation}
As a result, we consider the following three cases to analyze the
optimal solution
\begin{itemize}
\item $\eta>\frac{1+h_{t}^{k}p_{t}^{k,*}}{\tau Wh_{t}^{k}}$: It can be
obtained that $1-\mu_{t}+\lambda_{t}>1$, i.e., $0\leq\mu_{t}<\lambda_{t}$.
Thus, we have $p_{t}^{k,*}=p_{\max},\forall t\in\mathcal{T}_{t_{i}^{\text{S}}}$.
\item $\eta<\frac{1+h_{t}^{k}p_{t}^{k,*}}{\tau Wh_{t}^{k}}$: It can be
obtained that $1-\mu_{t}+\lambda_{t}<1$, i.e., $0\leq\lambda_{t}<\mu_{t}$.
Thus, we have $p_{t}^{k,*}=0,\forall t\in\mathcal{T}_{t_{i}^{\text{S}}}$.
\item $\eta=\frac{1+h_{t}^{k}p_{t}^{k,*}}{\tau Wh_{t}^{k}}$: It can be
obtained that $1-\mu_{t}+\lambda_{t}<1$, i.e., $\lambda_{t}=\mu_{t}$.
If $p_{t}^{k,*}=0$, we have $0\leq\mu_{t}<\lambda_{t}$, which is
in conflict with $\lambda_{t}=\mu_{t}$; If $p_{t}^{k,*}=p_{\max}$,
we have $0\leq\lambda_{t}<\mu_{t}$, which is also contradictory.
Hence, it can be obtained $0<p_{t}^{k,*}<1,\forall t\in\mathcal{T}_{t_{i}^{\text{S}}}$.
\end{itemize}
For the second case, if $p_{t}^{k,*}=0,\forall t\in\mathcal{T}_{t_{i}^{\text{S}}}$,
we have $\tau\sum_{t=t_{i}^{\text{S}}}^{t_{i}^{\text{S}}+t_{i}^{\text{T}}}R_{t}^{k}\left(\mathbf{p}_{t_{i}^{\text{S}}}^{k}\right)=0$,
which is in conflict with (\ref{eq:QoS for optimal solution}). Therefore,
we have $\eta\geq\frac{1+h_{t}^{k}p_{t}^{k,*}}{\tau Wh_{t}^{k}}$,
i.e., For the third case, we have $\tau\sum_{t=t_{i}^{\text{S}}}^{t_{i}^{\text{S}}+t_{i}^{\text{T,\ensuremath{\max}}}}R_{t}^{k}\left(\mathbf{p}_{t_{i}^{\text{S}}}^{k}\right)=D$
due to $\eta>0$ and (\ref{eq:L function to QoS}). Hence, it can
be obtained 
\begin{equation}
\tau\sum_{t=t_{i}^{\text{S}}}^{t_{i}^{\text{S}}+t_{i}^{\text{T,\ensuremath{\max}}}}W\log\left(\eta\tau Wh_{t}^{k}\right)=D.
\end{equation}
Then, the Lagrange multipliers is given as
\begin{equation}
\eta=e^{\frac{\frac{D}{\tau W}-\sum_{t=t_{i}^{\text{S}}}^{t_{i}^{\text{S}}+t_{i}^{\text{T,\ensuremath{\max}}}}\left(\tau Wh_{t}^{k}\right)}{t_{i}^{\text{T,\ensuremath{\max}}}}}.\label{eq:eta}
\end{equation}
Substituting (\ref{eq:eta}) into $\eta=\frac{1+h_{t}^{k}p_{t}^{k,*}}{\tau Wh_{t}^{k}}$,
we can obtain
\begin{equation}
p_{t}^{k,*}=\tau We^{\frac{\frac{D}{\tau W}-\sum_{t=t_{i}^{\text{S}}}^{t_{i}^{\text{S}}+t_{i}^{\text{T,\ensuremath{\max}}}}\log\left(\tau Wh_{t}^{k}\right)}{t_{i}^{\text{T,\ensuremath{\max}}}}}-\frac{1}{h_{t}^{k}}.
\end{equation}
The theorem is thus proved.

\subsection{Proof of Theorem \ref{thm:decrease with time}\label{Proof:DecreasewithTime}}

Denote $\mathbf{p}_{t_{i}^{\text{S}}}^{k,*}$ as the optimal solution
of problem (P2), while $\mathbf{p}_{t_{i}^{\text{S}}}^{k,\prime}$
is denoted as an another solution. Assume $p_{t}^{k,\prime}=p_{t}^{k,*},\forall t\in\left(t_{i}^{\text{S}},...,t_{i}^{\text{S}}+t_{i}^{\text{T,\ensuremath{\max}}}\right)/\left\{ t_{1},t_{2}\right\} $,
where $0<p_{t_{1}}^{k,*}<p_{\max}$ and $0<p_{t_{2}}^{k,*}<p_{\max}$
with $t_{1},t_{2}\in\left(t_{i}^{\text{S}},...,t_{i}^{\text{S}}+t_{i}^{\text{T,\ensuremath{\max}}}\right)$.
Without loss of generality, let $p_{t_{1}}^{k,\prime}=0$ and $p_{t_{2}}^{k,\prime}=p_{t_{2}}^{k,*}+\triangle p$
with $\triangle p>0$. In other words, the transmission time of solution
$\mathbf{p}_{t_{i}^{\text{S}}}^{k,\prime}$ is shorter than that of
solution $\mathbf{p}_{t_{i}^{\text{S}}}^{k,*}$. As mentioned in Proof
\ref{Proof:optimalsolution}, the optimal solution satisfies $\tau\sum_{t=t_{i}^{\text{S}}}^{t_{i}^{\text{S}}+t_{i}^{\text{T,\ensuremath{\max}}}}R_{t}^{k}\left(\mathbf{p}_{t_{i}^{\text{S}}}^{k,*}\right)=D$.
Therefore, we can obtain 
\begin{align}
 & \log\left(1+\left(p_{t_{2}}^{k,*}+\triangle p\right)h_{t_{2}}^{k}\right)-\log\left(1+p_{t_{2}}^{k,*}h_{t_{2}}^{k}\right)\nonumber \\
\geq & \log\left(1+p_{t_{1}}^{k,*}h_{t_{1}}^{k}\right).\label{eq:D of rate}
\end{align}
Moreover, due to $0<p_{t_{2}}^{k,*}<p_{\max}$, we have $\eta=\frac{1+h_{t_{2}}^{k}p_{t_{2}}^{k,*}}{\tau Wh_{t_{2}}^{k}}$
according to the third case in Proof \ref{Proof:optimalsolution}.
Hence, (\ref{eq:D of rate}) can be written as 
\begin{align}
 & \log\left(\eta\tau Wh_{t_{2}}^{k}+\triangle ph_{t_{2}}^{k}\right)-\log\left(\eta\tau Wh_{t_{2}}^{k}\right)\nonumber \\
\geq & \log\left(1+p_{t_{1}}^{k,*}h_{t_{1}}^{k}\right).\label{eq:D of rate-1}
\end{align}
We can obtain
\[
\log\left(1+\frac{\triangle p}{\eta\tau W}\right)\geq\log\left(1+p_{t_{1}}^{k,*}h_{t_{1}}^{k}\right),
\]
which can be rewritten as
\[
\triangle p\geq p_{t_{1}}^{k,*}h_{t_{1}}^{k}\eta\tau W.
\]
Similarly, $p_{t_{1}}^{k,*}$ and $h_{t_{1}}^{k}$ satisfy $\eta=\frac{1+h_{t_{1}}^{k}p_{t_{1}}^{k,*}}{\tau Wh_{t_{1}}^{k}}$.
Thus, we have $\triangle p\geq\left(1+h_{t_{1}}^{k}p_{t_{1}}^{k,*}\right)p_{t_{1}}^{k,*}\geq p_{t_{1}}^{k,*}$,
i.e., $\sum_{t=t_{i}^{\text{S}}}^{t_{i}^{\text{S}}+t_{i}^{\text{T,\ensuremath{\max}}}}p_{t_{1}}^{k,*}\leq\sum_{t=t_{i}^{\text{S}}}^{t_{i}^{\text{S}}+t_{i}^{\text{T,\ensuremath{\max}}}}p_{t_{1}}^{k,\prime}$.
The theorem is thus proved.

\subsection{Proof of Theorem \ref{thm:optimal time}\label{Proof:OptimalTime}}

The objective function of problem (P3) can be written as
\begin{align}
 & \underset{i\in\mathcal{I}_{t^{\prime}}}{\sum}E_{i}\left(t_{i}^{\text{T,\ensuremath{\max}}}\right)\nonumber \\
= & \underset{i\in\mathcal{I}_{t^{\prime}}}{\sum}\sum_{t=t_{i}^{\text{S}}}^{t_{i}^{\text{S}}+t_{i}^{\text{T,\ensuremath{\max}}}}\left(\tau We^{\frac{\frac{D_{i}}{\tau W}-\sum_{t=t_{i}^{\text{S}}}^{t_{i}^{\text{S}}+t_{i}^{\text{T,\ensuremath{\max}}}}\log\left(\tau Wh_{t}^{k}\right)}{t_{i}^{\text{T,\ensuremath{\max}}}}}-\frac{1}{h_{t}^{k}}\right).\label{eq:E1 in proof3}
\end{align}
In (\ref{eq:E1 in proof3}), $\sum_{i\in\mathcal{I}_{t^{\prime}}}\sum_{t=t_{i}^{\text{S}}}^{t_{i}^{\text{S}}+t_{i}^{\text{T,\ensuremath{\max}}}}\frac{1}{h_{t}^{k}}$
is constant due to (\ref{eq:P3-C2}) and (\ref{eq:P3-C3}). Thus,
our primary objective is to minimize the subtracter on the right-hand
side in (\ref{eq:E1 in proof3}). Denote $N_{I}$ as the number of
tasks in set $\mathcal{I}_{t^{\prime}}$. It can be obtained 
\begin{align}
 & \underset{i\in\mathcal{I}_{t^{\prime}}}{\sum}\sum_{t=t_{i}^{\text{S}}}^{t_{i}^{\text{S}}+t_{i}^{\text{T,\ensuremath{\max}}}}\tau We^{\frac{\frac{D_{i}}{\tau W}-\sum_{t=t_{i}^{\text{S}}}^{t_{i}^{\text{S}}+t_{i}^{\text{T,\ensuremath{\max}}}}\log\left(\tau Wh_{t}^{k}\right)}{t_{i}^{\text{T,\ensuremath{\max}}}}}\nonumber \\
\geq & N_{I}\sqrt{\underset{i\in\mathcal{I}_{t^{\prime}}}{\prod}\left(\sum_{t=t_{i}^{\text{S}}}^{t_{i}^{\text{S}}+t_{i}^{\text{T,\ensuremath{\max}}}}\tau We^{\frac{\frac{D_{i}}{\tau W}-\sum_{t=t_{i}^{\text{S}}}^{t_{i}^{\text{S}}+t_{i}^{\text{T,\ensuremath{\max}}}}\log\left(\tau Wh_{t}^{k}\right)}{t_{i}^{\text{T,\ensuremath{\max}}}}}\right)}\label{eq:E2 in proof3}
\end{align}
The condition for equality of (\ref{eq:E2 in proof3}) is

\begin{align}
 & \sum_{t=t_{n}^{\text{S}}}^{t_{n}^{\text{S}}+t_{n}^{\text{T,\ensuremath{\max}}}}e^{\frac{\frac{D_{n}}{\tau W}-\sum_{t=t_{n}^{\text{S}}}^{t_{n}^{\text{S}}+t_{n}^{\text{T,\ensuremath{\max}}}}\log\left(\tau Wh_{t}^{k}\right)}{t_{n}^{\text{T,\ensuremath{\max}}}}}\nonumber \\
= & \sum_{t=t_{m}^{\text{S}}}^{t_{m}^{\text{S}}+t_{m}^{\text{T,\ensuremath{\max}}}}e^{\frac{\frac{D_{m}}{\tau W}-\sum_{t=t_{m}^{\text{S}}}^{t_{m}^{\text{S}}+t_{m}^{\text{T,\ensuremath{\max}}}}\log\left(\tau Wh_{t}^{k}\right)}{t_{m}^{\text{T,\ensuremath{\max}}}}},\forall n,m\in\mathcal{I}_{t^{\prime}},\label{eq:E3 in proof3}
\end{align}
which can be further written as 
\begin{align}
 & t_{n}^{\text{T,\ensuremath{\max}}}e^{\frac{\frac{D_{n}}{\tau W}-\sum_{t=t_{n}^{\text{S}}}^{t_{n}^{\text{S}}+t_{n}^{\text{T,\ensuremath{\max}}}}\log\left(\tau Wh_{t}^{k}\right)}{t_{n}^{\text{T,\ensuremath{\max}}}}}\nonumber \\
= & t_{m}^{\text{T,\ensuremath{\max}}}e^{\frac{\frac{D_{m}}{\tau W}-\sum_{t=t_{m}^{\text{S}}}^{t_{m}^{\text{S}}+t_{m}^{\text{T,\ensuremath{\max}}}}\log\left(\tau Wh_{t}^{k}\right)}{t_{m}^{\text{T,\ensuremath{\max}}}}},\forall n,m\in\mathcal{I}_{t^{\prime}}.\label{eq:E3 in proof3-1}
\end{align}
Thus, we have
\begin{align}
 & \frac{t_{n}^{\text{T,\ensuremath{\max}}}}{t_{m}^{\text{T,\ensuremath{\max}}}}\nonumber \\
= & \frac{e^{\frac{D_{m}}{\tau Wt_{m}^{\text{T,\ensuremath{\max}}}}}e^{\frac{1}{t_{n}^{\text{T,\ensuremath{\max}}}}\sum_{t=t_{n}^{\text{S}}}^{t_{n}^{\text{S}}+t_{n}^{\text{T,\ensuremath{\max}}}}\log\left(\tau Wh_{t}^{k}\right)}}{e^{\frac{D_{n}}{\tau Wt_{n}^{\text{T,\ensuremath{\max}}}}}e^{\frac{1}{t_{m}^{\text{T,\ensuremath{\max}}}}\sum_{t=t_{m}^{\text{S}}}^{t_{m}^{\text{S}}+t_{m}^{\text{T,\ensuremath{\max}}}}\log\left(\tau Wh_{t}^{k}\right)}},\forall n,m\in\mathcal{I}_{t^{\prime}}.\label{eq:E4 in proof3}
\end{align}

\end{appendix}

\bibliographystyle{IEEEtran}
\bibliography{TRL}

\end{document}